\documentclass[letterpaper, 10 pt, conference]{ieeeconf}  % Comment this line out if you need a4paper

\IEEEoverridecommandlockouts                              % This command is only needed if 
\usepackage{amsmath}
\usepackage{amssymb}

\usepackage{amsthm}
\usepackage{mathtools}
\mathtoolsset{showonlyrefs=true}
\usepackage{cancel}
\usepackage{cases}
\usepackage{graphicx}
\graphicspath{{fig/}}
\usepackage[caption=false,font=footnotesize]{subfig}
\usepackage{multirow}
\usepackage{color}
\let\labelindent\relax
\usepackage{enumitem}

\newtheorem{definition}{Definition}
\newtheorem{proposition}{Proposition}
\newtheorem{theorem}{Theorem}
\newtheorem{remark}{Remark}
\newtheorem{example}{Example}
\DeclareMathOperator*{\minimize}{minimize~}
\DeclareMathOperator*{\subjto}{subject\,to~}
\newcommand{\R}{\mathbb R}
\newcommand{\mc}{\mathcal}
\newcommand{\tr}{^T}

\title{\LARGE \bf
Resilience and Energy-Awareness\\in Constraint-Driven-Controlled Multi-Robot Systems
}

\author{Gennaro Notomista% <-this % stops a space
\thanks{
	\copyright 2022 IEEE.  Personal use of this material is permitted.  Permission from IEEE must be obtained for all other uses, in any current or future media, including reprinting/republishing this material for advertising or promotional purposes, creating new collective works, for resale or redistribution to servers or lists, or reuse of any copyrighted component of this work in other works.
}
\thanks{G. Notomista is with the Department of Electrical and Computer Engineering, University of Waterloo, Waterloo, ON, Canada {\tt\small gennaro.notomista@uwaterloo.ca}}%
}

\begin{document}

\maketitle
\thispagestyle{empty}
\pagestyle{empty}

%%%%%%%%%%%%%%%%%%%%%%%%%%%%%%%%%%%%%%%%%%%%%%%%%%%%%%%%%%%%%%%%%%%%%%%%%%%%%%%%
\begin{abstract}

In the context of constraint-driven control of multi-robot systems, in this paper, we propose an optimization-based framework that is able to ensure resilience and energy-awareness of teams of robots. The approach is based on a novel, frame-theoretic, measure of resilience which allows us to analyze and enforce resilient behaviors of multi-robot systems. The properties of resilience and energy-awareness are encoded as constraints of a convex optimization program which is used to synthesize the robot control inputs. This allows for the combination of such properties with the execution of coordinated tasks to achieve resilient and energy-aware robot operations. The effectiveness of the proposed method is illustrated in a simulated scenario where a team of robots is deployed to execute two tasks subject to energy and resilience constraints.

\end{abstract}

%%%%%%%%%%%%%%%%%%%%%%%%%%%%%%%%%%%%%%%%%%%%%%%%%%%%%%%%%%%%%%%%%%%%%%%%%%%%%%%%
\section{Introduction}

Multi-robot systems are rapidly moving from the curated environments of academic laboratories to the real world. Applications in which multi-robot systems have already demonstrated or promise to be particularly advantageous include environmental monitoring \cite{roldan2016heterogeneous}, precision agriculture \cite{zhang2017development}, environment exploration \cite{wurm2008coordinated}, search and rescue \cite{baxter2007multi} (see also the surveys \cite{rizk2019cooperative,yan2013survey,cortes2017coordinated} and references therein). In these scenarios, robot teams are generally employed over long time horizons, therefore, in order to guarantee their successful deployment, we need to consider design and control principles pertaining the discipline known as long-duration autonomy \cite{steinberg2016long,notomista2020long}.

Resilience and energy-awareness are two fundamental properties that robotic systems need to exhibit in order to be successfully deployed in the real world over long periods of time. These properties are complementary in the following sense: under nominal conditions, ensuring that enough energy is available at each point in time is a necessary condition to ensure that the robots can sustain themselves over long time horizons. When conditions are not nominal---because the environment is unknown, unstructured, non-stationary, or because functionalities have been lost due to system failures---resilience is the property that allows robots to react and recover by structurally changing their behavior or their objective in order to survive, in a robotic sense. Robotic systems that account for the energy spent for motion, communication, and computation, while executing the desired tasks, have been extensively studied. Prominent applications include persistent environmental monitoring \cite{smith2011persistent,notomista2020persistification}, energy autonomy \cite{fouad2020energy}, motion and communication energy co-optimization \cite{yan2014go,ali2018motion}.

Resilience is experiencing an increasing interest \cite{prorok2020robust,yu2021resilient,mayya2021resilient,notomista2021resilient} (see also the recent survey \cite{prorok2021beyond}). Differently from robustness---the property that characterizes systems which typically try to achieve a desired performance even under (bounded) disturbances, as it happens in robust control---or adaptivity---owing to which systems adapt some of their parameters in order to accomplish a desired task, as it is the case in adaptive control---by resilience we mean the property thanks to which robotic systems are able to \textit{recover from failure by altering their behavior and/or their objective}. So far resilient multi-robot systems share distinctive features with robust and adaptive systems. The approaches proposed in \cite{leblanc2013resilient,wehbe2021probabilistic,wardega2019resilience,saldana2017resilient}, for instance, belong to the former category, and consider resilience of robot teams to communication failures, attacks, and non-cooperative robots. More oriented to system adaptiveness are the recent works \cite{ramachandran2019resilience,chamon2020resilient,notomista2021resilient}, which achieve resilience by reacting to resource failure, endogenous and exogenous disturbances.

In this paper, we propose a new way of quantifying and achieving resilience, which is suitable for the constraint-driven coordinated control of multi-robot systems \cite{notomista2019constraint}. The effectiveness of this control strategy---which stems from the fact that it allows us to consider task execution and energy holistically as constraints of a minimum-control-effort optimization program---has been demonstrated through a number of applications \cite{notomista2019optimal,cabral2020autonomous,beaver2021energy}. In particular, its amenability for long-duration autonomy applications has been shown in \cite{notomista2019constraint}, where energy constraints are explicitly considered to render multi-robot coordinated tasks persistent over long time horizons. One of the simplest forms of constraint-driven control can be found in \cite{freeman1996inverse}, where a minimum-control-effort optimization problem is proposed to synthesize a stabilizing controller for dynamical systems. Since then, optimization has played a more and more central role in the control synthesis for dynamical systems. Some prominent examples can be found in the context of model predictive control \cite{borrelli2017predictive}, optimal control \cite{craven1998control}, Lyapunov-based methods \cite{prajna2004nonlinear}, affine controllers \cite{skaf2010design}, convex optimization control policies \cite{agrawal2020learning,bertsekas2000dynamic}, multi-task control \cite{notomista2019constraint}.

The main contributions of this paper are the following:
\begin{enumerate}[label=(\roman*)]
\item Providing a novel frame-theoretic metric to assess the resilience of constraint-driven-controlled multi-robot systems; compared to other approaches, such as \cite{saldana2017resilient}, the proposed metric allows us to quantify resilience with respect to the concurrent execution of multiple tasks;
\item Proposing a control strategy to improve the resilience of multi-robot systems;
\item Showing how resilience and energy-awareness for multi-robot systems can be combined using the constraint-driven control paradigm in order to achieve long-duration robot autonomy.
\end{enumerate}
The novel resilience metric we propose finds its roots in the theory of frames \cite{waldron2018introduction} and leverages a recent characterization of the so-called finite normalized tight frames in terms of a frame potential \cite{benedetto2003finite}. Frame potentials are then used to define the \textit{resilience constraint}, i.e. an additional constraint that is able to improve the resilience of multi-robot systems.

The remainder of the paper is organized as follows. Section~\ref{sec:background} recalls the energy-aware constraint-driven control formulation and introduces the frame-theoretic tools employed in this paper to analyze and synthesize resilient multi-robot systems. Section~\ref{sec:resilientframes} is devoted to the definition of the novel resilience metric and presents a way to improve resilience of multi-robot systems amenable for the constraint-driven control formulation adopted in this paper. In Section~\ref{sec:experiments}, the results of simulated experiments are reported. Section~\ref{sec:conclusions} concludes the paper.

\section{Background}
\label{sec:background}

\subsection{Constraint-Driven Control of Multi-Robot Systems}

The simplest form of the constraint-driven control of one robot to execute one task can be expressed by the following optimization program, solved point-wise in time:
\begin{equation}
	\label{eq:cdc1robot1task}
	\begin{aligned}
		\minimize_u &\|u\|^2 \\
		\subjto &c_{task}(x,u)\le0
	\end{aligned}
\end{equation}
where $x\in\R^{n_x}$ is the robot state, $u\in\R^{n_u}$ its control input, and $c_{task}(x,u)\le0$ encodes the execution of a desired task. This idea can be extended to $M$ tasks executed by $N$ robots as follows:
\begin{equation}
	\label{eq:cdcnrobotsmtasks}
	\begin{aligned}
		\minimize_{u_i,\delta_i} & \|u_i\|^2 + \kappa \|\delta_i\|^2\\
		\subjto & c_{task,ij}(x,u_i)\le\delta_{ij}\quad \forall j\in\{1,\ldots,M\}
	\end{aligned}
\end{equation}
The components of the vector $\delta_i\in\R^M$ are used to denote slack variables on each of the tasks executed by the robots. $\delta_{ij}$ represents the slack on task $j$ executed by robot $i$. The presence of slack variables is what allows the system to prioritize the desired behaviors: hard constraints (as can be, for instance, energy control and collision avoidance) are not relaxed---their $\delta_{ij}$ are equal to 0---whereas slacked constraints can be used to encode low-priority tasks.

A rich and expressive way of encoding robotic tasks is by means of sets of the robot state space which are to be rendered asymptotically stable or forward invariant (safe). To ensure stability and safety of such sets, control barrier functions (CBFs) can be employed \cite{ames2019control}. In this paper, we assume that robots can be modeled using the control affine dynamical system
\begin{equation}
\label{eq:ctrlaffinerobotdyn}
\dot x_i = f_i(x_i) + g_i(x_i) u_i,
\end{equation}
where $x_i\in\R^{n_{x_i}}$ is the state of robot $i$, $u_i\in\R^{n_{u_i}}$ its input, $f_i\colon\R^{n_{x_i}}\to\R^{n_{x_i}}$ and $g_i\colon\R^{n_{x_i}}\to\R^{n_{x_i}\times n_{u_i}}$ are locally Lipschitz continuous vector fields. Following the formulation in \cite{notomista2019constraint}, we can express the constraint function as
\begin{equation}
\label{eq:ctaskcbf}
c_{task,ij}(x,u_i) := -L_{f_i} h_{ij}(x) - L_{g_i} h_{ij}(x) u_i - \alpha_i(h_{ij}(x)),
\end{equation}
where $h_{ij}$ is a CBF, and the expressions $L_{f_i}h_{ij}(x) = \frac{\partial h_{ij}}{\partial x} f_i(x)$ and $L_{g_i}h_{ij}(x) = \frac{\partial h_{ij}}{\partial x} g_i(x)$ denote the Lie derivatives of $h_{ij}$ in the directions of $f_i$ and $g_i$. Enforcing the constraint $c_{task,ij}(x,u_i)\le0$ ensures that the set $\mc S_j = \cap_{i=1}^N\{ x_i \colon h_{ij}(x)\ge0 \}$ is asymptotically stable or forward invariant \cite{ames2019control}.

As far as energy awareness is concerned, it has been shown in \cite{notomista2020long} how controlling the energy levels in the battery of the robots can also be expressed as the forward invariance property of a given subset of the state space of the robots. Let us define the energy CBF as
\begin{equation}
\label{eq:energycbf}
h_{e,i}(x_i,e_i) = e_i - e_{min} - \alpha_c\left( \|p(x_i)-p_{c,i}\| \right),
\end{equation}
where $e_i$ is the energy stored in the battery of robot $i$ and $e_{min}$ is a minimum threshold above which we want $e_i$ to remain. The value $p(x_i)$ denotes the robot position in space (e.g., $p(x_i)\in\R^2$ for planar robots, $p(x_i)\in\R^3$ for aerial robots), and $p_{c,i}$ is the spatial location of a charging station, i.e. a place that the robot can reach in order to recharge its battery. The function $\alpha_c$ is a monotonically increasing function, so that the quantity $\alpha_c\left( \|p(x_i)-p_{c,i}\| \right)$ is an upper bound on the energy required to reach the charging station located at $p_c$ from location $p(x)$ (see \cite{notomista2020long} for the detailed derivation and analysis).

As a result, the constraint that the energy in the batteries of the robots always stays above a minimum threshold can be expressed as follows:
\begin{equation}
	\label{eq:energyconstraint}
	-L_{f_i} h_{e,i}(x_i,e_i) - L_{g_i}h_{e,i}(x_i,e_i) u_i - \alpha_e(h_{e,i}(x_i,e_i)) \le 0,
\end{equation}
where
$h_{e,i}$ is given by \eqref{eq:energycbf} and
$\alpha_e\colon\R\to\R$ is a Lipschitz continuous extended class $\mc K$ function. Thus, with the CBF notation introduced above, we can integrate energy constraints in \eqref{eq:cdcnrobotsmtasks} as follows:
\begin{equation}
	\label{eq:cdcnrobotsmtaskscbfs}
	\begin{aligned}
		\minimize_{u,\delta} & \|u\|^2 + \kappa \|\delta\|^2\\
		\subjto & \sum_{i=1}^N \big( -L_{f_i} h_{ij}(x) - L_{g_i} h_{ij}(x) u_i \\
		&\quad -\alpha_i(h_{ij}(x)) \big) \le \sum_{i=1}^N \delta_{ij}\quad \forall j\in\{1,\ldots,M\}\\[0.5\baselineskip]
		&-L_{f_i} h_{e,i}(x_i,e_i) - L_{g_i}h_{e,i}(x_i,e_i) u_i\\
		&\quad -\alpha_e(h_{e,i}(x_i,e_i)) \le 0\quad \forall i\in\{1,\ldots,N\},
	\end{aligned}
\end{equation}
where $x=[x_1\tr,\ldots,x_N\tr]\tr$, $u=[u_1\tr,\ldots,u_N\tr]\tr$, $\delta=[\delta_1\tr,\ldots,\delta_N\tr]\tr\in\R^{NM}$ denote the ensemble state, control inputs, and slack variables, for the multi-robot system.
The following example shows how the constraint-driven control paradigm can be used to let a multi-robot system execute coordinated control tasks such as the consensus protocol \cite{cortes2017coordinated}.
\begin{example}
Consider the set
\begin{equation}
S_1 = \cap_{i=1}^N\left\{ x_i \colon h_{i1}(x) \ge0 \right\},
\end{equation}
where
\begin{equation}
	\label{eq:hi1}
	h_{i1}(x) = -\sum_{k=1}^{N} \|x_i-x_k\|^2.
\end{equation}
Then, the optimal $u_i$ solution of \eqref{eq:cdc1robot1task}, where $c_{task}$ is given by \eqref{eq:ctaskcbf} and $h_{ij}$ is $h_{i1}$ in \eqref{eq:hi1}, lets the robots execute the consensus protocol. See \cite{notomista2019constraint} for more examples of the coordinated control of multi-robot systems.
\end{example}

\subsection{Frames and Frame Potentials}
\label{subsec:frames}

In this paper, we propose a frame-theoretic metric of resilience of constraint-driven-controlled multi-robot systems. To this end, in this section, we give a brief overview on the main concepts of frame theory used in the remainder of the paper.

Frames play a fundamental role in several areas, including analog-to-digital conversion, compressed sensing, phaseless reconstruction, transmission with erasures (see in \cite{strohmer2003grassmannian,benedetto2003finite,casazza2006physical,bodmann2015frame} and references therein). Compared to many applications of frames, where redundancy of sets of vectors that form a frame is leveraged to improve robustness with respect to disturbances of various nature, in this paper we make use of a different property of frames, and in particular, of the so-called finite normalized tight frames. This property consists in the characterization of such frames by means of a suitably defined frame potential, as explained in detail in the following. First let us start by defining a frame.

\begin{definition}[Tight normalized frame]
	A set of vectors $\{ v_i \}_{i=1}^d$ in a $n$-dimensional real Hilbert space $H$ is a frame if there exist constants $0<A\le B<\infty$ such that
	\begin{equation}
	A \| y \|^2 \le \sum_{i=1}^d \langle y, v_i \rangle^2 \le B \| y \|^2
	\end{equation}
	for all $y\in H$.
	
	A frame $\{ v_i \}_{i=1}^d$ is $A$-tight if there exists a constant $A>0$ so that
	\begin{equation}
	\sum_{i=1}^d \langle y, v_i \rangle^2 = A \| y \|^2
	\end{equation}
	for all $y\in H$.
	
	A tight frame $\{ v_i \}_{i=1}^d$ is normalized if $\|v_i\|=1$ for all $i$.
\end{definition}
In the remainder of this paper, finite normalized tight frames will be abbreviated as FNTFs. The following theorem offers a characterization of FNTFs which will be useful in the next section to quantify and achieve resilient behaviors of constraint-driven-controlled multi-robot systems.

\begin{theorem}[From Theorem 7.1 in \cite{benedetto2003finite}]
	\label{thm:fp}
	For a given $d$ and $n$, consider the frame potential
	\begin{equation}
		\label{eq:fpthm}
			FP \colon S(\R^d)^n \to [0,\infty) \colon \{ v_i \}_{i=1}^d \to \sum_{i=1}^d \sum_{j=1}^d \langle v_i, v_j \rangle^2
	\end{equation}
	where $S(\R^d)$ denotes the unit sphere in $\R^d$. Then
	\begin{enumerate}[label=(\roman*)]
		\item Every local minimizer of the frame potential is also a global minimizer.
		\item If $n\ge d$, the minimum value of FP is $n^2/d$, and the minimizers are precisely the FNTFs for $\R^d$.
	\end{enumerate}
\end{theorem}

In the next section, we will show how FNTFs and frame potentials can be used to quantify the resilience of multi-robot systems controlled using the constraint-driven control paradigm. Moreover, we will derive a resilience constraint which can be enforced in order to improve the resilience properties of robot teams.

\section{Resilience Analysis and Synthesis}
\label{sec:resilientframes}

In this section, we draw the connection between frame potentials and the resilience of multi-robot systems controlled using the constraint-driven control paradigm \eqref{eq:cdcnrobotsmtaskscbfs}, showing how the former can be interpreted as a suitable metric for the latter. Then, leveraging this connection, we show how to synthesize controllers to allow multi-robot systems to execute tasks in a resilient fashion.

\subsection{Resilience Analysis}
\label{subsec:analysis}

First, we recall that, differently from robustness and adaptivity, we would like a resilient system to be able to react and recover by structurally changing its behavior or its objective in order to survive in response to a disturbance, which can be modeled or unmodeled, known or unknown, exogenous or endogenous---i.e. caused by the system itself or by the environment in which it is deployed.

Therefore, let us consider the case when the optimal $u_k$, denoted by $u_k^\star$, solution of \eqref{eq:cdcnrobotsmtaskscbfs} cannot be executed by robot $k$ for some $k$. This can model situations where robot $k$'s actuators are malfunctioning (\textit{actuation disturbance}), or robot $k$'s computational capabilities are compromised and therefore it cannot evaluate its optimal control input (\textit{computational disturbance}), or, in case the optimal controller is evaluated at a central computational unit, this can represent the case where the calculated optimal $u_k$ cannot be transferred from the central computational unit to robot $k$ (\textit{communication disturbance}). In such situations, we would like the robots to drive towards configurations that are less affected by these disturbances.

In order to characterize these amenable configurations, let us consider the expected mismatch, $\Delta$, between the optimal multi-robot behavior and the executed one. This mismatch can be measured by the effect that the control input $u_i$ has on the tasks to execute, as follows:
\begin{equation}
	\label{eq:expectation}
	\begin{aligned}
		\Delta = &\mathop{\mathbb{E}}_{y\in S(\R^M)} \sum_{k=1}^N \Bigg\| \sum_{i=1}^N \langle L_{g_i} h_{i}(x) u_i^\star, y \rangle L_{g_i} h_{i}(x) u_i^\star \\
		&- \sum_{\substack{i=1\\i\neq k}}^N \langle L_{g_i} h_{i}(x) u_i^\star, y \rangle L_{g_i} h_{i}(x) u_i^\star \Bigg\|^2,
	\end{aligned}
\end{equation}
where
\begin{equation}
	L_{g_i} h_{i}(x) = \begin{bmatrix}
		L_{g_i} h_{i1}(x)\\
		\vdots\\
		L_{g_i} h_{iM}(x)
	\end{bmatrix}
\end{equation}
and
\begin{equation}
	y = \frac{\hat y}{\|\hat y\|}, \text{~with~} \hat y = \begin{bmatrix}
		\sum\limits_{i=1}^N \big( -L_{f_i} h_{i1}(x) - \alpha_i(h_{i1}(x)) \big)\\
		\vdots\\
		\sum\limits_{i=1}^N\big( -L_{f_i} h_{iM}(x) - \alpha_i(h_{iM}(x)) \big)
	\end{bmatrix}.
\end{equation}

The expectation in \eqref{eq:expectation} can be simplified as follows:
\begin{equation}
	\label{eq:resilience}
	\begin{aligned}
		\Delta=&\int\limits_{y\in S(\R^M)} \sum_{k=1}^N \left\| \langle L_{g_k} h_{k}(x) u_k^\star, y \rangle L_{g_i} h_{k}(x) u_k^\star \right\|^2\\
		=&\int\limits_{y\in S(\R^M)} \sum_{k=1}^N \langle L_{g_k} h_{k}(x) u_k^\star, y \rangle ^2 \left\|L_{g_i} h_{k}(x) u_k^\star \right\|^2\\
		=&\int\limits_{y\in S(\R^M)} \sum_{k=1}^N \langle L_{g_k} h_{k}(x) u_k^\star, y \rangle ^2 \underbrace{\left\|L_{g_i} h_{k}(x) u_k^\star \right\|^2}_{=1}\\
		=&\int\limits_{y\in S(\R^M)} A \underbrace{\|y\|^2}_{=1} = A\left| S(\R^M) \right|
	\end{aligned}
\end{equation}
for a normalized $A$-tight frame. By Theorem 2.1(b) in \cite{benedetto2003finite}, $A\ge1$, therefore the minimum is achieved when $A=1$. Thus, if $\left\{ L_{g_i} h_{i}(x) \right\}_{i=1}^N$ is a FNTF, then the expectation of the mismatch is minimized.

By Theorem~\ref{thm:fp}, FNTFs can be characterized in terms of a frame potential. Then, we propose to use the value of a suitably-defined frame potential as a measure of the resilience of the multi-robot system controlled using the constraint-driven control paradigm. Based on \eqref{eq:fpthm}, the frame potential suitable to quantify resilience in constraint-driven-controlled multi-robot systems is the following:
\begin{equation}
	\label{eq:fp}
	\begin{aligned}
		FP &\colon \{ L_{g_i} h_{i}(x) \}_{i=1}^N \\
		&\to \sum_{i=1}^N \sum_{j=1}^N
		\sum_{k=1}^{n_{u,i}} \sum_{l=1}^{n_{u,j}} \frac{\langle [L_{g_i} h_{i}(x)]_k, [L_{g_j} h_{j}(x)]_l \rangle^2}{\|[L_{g_i} h_{i}(x)]_k\|\|[L_{g_j} h_{j}(x)]_l\|},
	\end{aligned}
\end{equation}
where the notation $[L_{g_i} h_{i}(x)]_k$ is used to denote the $k$-th column of $L_{g_i} h_{i}(x)$.

\begin{remark}
	The higher the value of the frame potential, the lower the resilience of the robotic system, as the more work---proportional to the potential difference \cite{benedetto2003finite}---is required to reconfigure the system in response to the effect of disturbances.
\end{remark}

\subsection{Resilience Synthesis}

The frame potential \eqref{eq:fp} defined in the previous section can be leveraged to achieve FNTFs. The function in \eqref{eq:fpthm}---based on that defined in \cite{benedetto2003finite}---is defined over $S(\R^d)$, i.e. it considers already normalized vectors. In the case of the frame potential \eqref{eq:fp}, the vectors $\{ L_{g_i} h_{i}(x) \}_{i=1}^N$ are not necessarily normalized, as they depend on the tasks that the robots need to execute. Therefore, the following modified frame potential is proposed:
\begin{equation}
\label{eq:fpr}
	\begin{aligned}
		FP_R &\colon \{ L_{g_i} h_{i}(x) \}_{i=1}^N \\
		&\to \sum_{i=1}^N \Bigg( \sum_{j=1}^N
		\sum_{k=1}^{n_{u,i}} \sum_{l=1}^{n_{u,j}} \frac{\langle [L_{g_i} h_{i}(x)]_k, [L_{g_j} h_{j}(x)]_l \rangle^2}{\|[L_{g_i} h_{i}(x)]_k\|\|[L_{g_j} h_{j}(x)]_l\|}\\
		&\qquad+		\sum_{k=1}^{n_{u,i}}\left( 1 - \|[L_{g_i} h_{i}(x)]_k\|^2 \right)^2 \Bigg)
	\end{aligned}
\end{equation}

\begin{proposition}
$\{ L_{g_i} h_{i}(x) \}_{i=1}^N$ is a minimizer of the frame potential defined in \eqref{eq:fpr} $\Leftrightarrow$ $\{ L_{g_i} h_{i}(x) \}_{i=1}^N$ is a FNTF.
\end{proposition}
\begin{proof}
($\Leftarrow$)
If $\{ L_{g_i} h_{i}(x) \}_{i=1}^N$ is a FNTF, by Theorem~\ref{thm:fp} the first part of the expression of the modified frame potential \eqref{eq:fpr}, equal to \eqref{eq:fp}, achieves its minimum. The second part of the modified frame potential, i.e.
\begin{equation}
\sum_{i=1}^N \sum_{k=1}^{n_{u,i}}\left( 1 - \|[L_{g_i} h_{i}(x)]_k\|^2 \right)^2 \ge 0
\end{equation}
is always non-negative and in the case of normalized frames its value is 0. Hence, the modified frame potential is minimized when $\{ L_{g_i} h_{i}(x) \}_{i=1}^N$ is a FNTF.

($\Rightarrow$) To show that a minimizer of \eqref{eq:fpr} is a FNTF, let us start by computing the gradient of \eqref{eq:fpr}. For ease of notation, let us rewrite the expression in \eqref{eq:fpr} as follows:
\begin{equation}
\begin{aligned}
FP_R(v) = &\sum_{i,j,k,l} \frac{\langle v_{ik}, v_{jl} \rangle^2}{\|v_{ik}\|\|v_{jl}\|} +	\left( 1 - \|v\|^2 \right)^2\\
= & FP\left(\left\{ \frac{v_{ik}}{\|v_{ik}\|} \right\}\right) + \left( 1 - \|v\|^2 \right)^2
\end{aligned}
\end{equation}
where $v_{ik} = [L_{g_i} h_{i}(x)]_k$, and $v$ is the stack of all $v_{ik}$. Then, the derivative of $FP_R$ with respect to $v$ evaluates to:
\begin{equation}
\begin{aligned}
	\frac{\partial FP_R}{\partial v}(v) &= \frac{\partial FP}{\partial v}\left(\frac{v}{\|v\|}\right) \frac{1}{\|v\|}\left(I - \frac{vv\tr}{\|v\|^2} \right)\\
	& -4 \left( 1 - \|v\|^2 \right) v\tr\\
	&= \underbrace{\frac{\partial FP}{\partial v}\left(\frac{v}{\|v\|}\right) \frac{1}{\|v\|}P_{v^\perp}}_{\perp v}-\underbrace{4 \left( 1 - \|v\|^2 \right) v\tr}_{\parallel v},
\end{aligned}
\end{equation}
where $P_{v^\perp}$ is the projector on the orthogonal complement of the set $\{v_{ik}\}_{i,k}$. As the two components of the gradient are orthogonal to each other, it follows that
\begin{equation}
\frac{\partial FP_R}{\partial v}(v) = 0 \implies \left\{ \begin{aligned}\frac{1}{\|v\|}P_{v^\perp} \frac{\partial FP}{\partial v}\left(\frac{v}{\|v\|}\right)\tr = 0\\
4 \left( 1 - \|v\|^2 \right) v = 0.
\end{aligned}\right.
\label{eq:gradientconditions}
\end{equation}
Except for the trivial case $v=0$, by Theorem~\ref{thm:fp}, the first condition in \eqref{eq:gradientconditions} is satisfied when $\left\{\frac{v}{\|v\|}\right\}$ is a FNTF, whereas the second condition holds for normalized frames $\{v_{ik}\}_{i,k}$. Hence, both are simultaneously satisfied when $\{v_{ik}\}_{i,k}$ is a FNTF.
\end{proof}

This proposition is what allows us to improve the resilience of multi-robot systems. In fact, using the constraint-control framework, we let
\begin{equation}
h_R(x) = -FP_R(\{ L_{g_i} h_{i}(x) \}_{i=1}^N)
\end{equation}
and define the following \textit{resilience constraint}:
\begin{equation}
\label{eq:rc}
-L_{f_i} h_R(x) - L_{g_i} h_R(x) u_i - \alpha_{R,i}(h_R(x)) \le 0,
\end{equation}
where $\alpha_{R,i}$ are Lipschitz continuous extended class $\mc K$ functions. Constraint \eqref{eq:rc}, if enforced in \eqref{eq:cdcnrobotsmtaskscbfs}, leads to a reconfiguration of the robots characterized by higher resilience values, as measured by the frame potential---we recall that lower values of $FP_R$ correspond to higher resilience properties quantified as in \eqref{eq:resilience}. The main optimization program solved to compute the controller required by $N$ robots to execute $M$ tasks in a resilient fashion is then the following:
\begin{equation}
	\label{eq:cdcnrobotsmtaskscbfsresilient}
	\begin{aligned}
		\minimize_{u,\delta\in\R^{N(M+1)}} & \|u\|^2 + \kappa \|\delta\|^2\\
		\subjto & \sum_{i=1}^N \big( -L_{f_i} h_{ij}(x) - L_{g_i} h_{ij}(x) u_i \\
		&\quad -\alpha_i(h_{ij}(x)) \big) \le \sum_{i=1}^N \delta_{ij}\quad \forall j\in\{1,\ldots,M\}\\
		&-L_{f_i} h_{e,i}(x_i,e_i) - L_{g_i}h_{e,i}(x_i,e_i) u_i\\
		&\quad -\alpha_e(h_{e,i}(x_i,e_i)) \le 0\quad \forall i\in\{1,\ldots,N\}\\
		&-L_{f_i} h_R(x) - L_{g_i} h_R(x) u_i\\
		&\quad -\alpha_{R,i}(h_R(x)) \le \delta_{i,M+1}\quad \forall i\in\{1,\ldots,N\},
	\end{aligned}
\end{equation}

\begin{remark}
The resilience constraint \eqref{eq:rc} is enforced in the optimization problem \eqref{eq:cdcnrobotsmtaskscbfsresilient} in addition to the constraints encoding the tasks. Therefore, the robots are continuously optimizing their resilience properties. For this reason, it might also be desirable to let the resilience constraint \eqref{eq:rc} slack based on the execution of the other tasks. This is obtained by introducing the slack variables $\delta_{i,M+1}$, and it allows us to prioritize the execution of the tasks over resilience. This way, only if tasks cannot be executed because of disturbances---or because it is too expensive from an energetic point of view, measured in terms of $\|u\|^2$---then a reconfiguration to improve resilience is performed.
\end{remark}

\begin{remark}
In general, the resilience constraint \eqref{eq:rc} cannot be decentralized in the sense that robot $i$ requires the knowledge of all other robots in the team, and not just of a suitably defined subset of them (its neighborhood), in order to evaluate its control input $u_i$.
If there are only two tasks, however, it can be shown that the potential minimization---corresponding to assembling a finite normalized tight frame in two dimensions---can be rendered decentralized. This two-task approach can be leveraged by considering the binary choice of ``doing \textsc{vs} not doing'' each given task. Thus, at the expenses of increasing the computational load for each robot, a decentralized algorithm can be obtained which requires no communication overhead.
Finally, it is worth noticing that, owing to the convexity of the optimization program \eqref{eq:cdcnrobotsmtaskscbfsresilient}, the computational complexity of solving for robot control inputs is polynomial in the number of robots and number of tasks.
\end{remark}

\section{Experimental Results}
\label{sec:experiments}

To showcase the effectiveness of the proposed resilience synthesis approach, in this section we consider a simulated scenario with a team of 6 ground mobile robot modeled using single integrator dynamics $\dot x_i = u_i$, with $x_i,u_i\in\R^2$ for all $i\in\{1,\ldots,6\}$. The energy dynamics have been modeled as in \cite{notomista2019constraint}. We want the robot team to perform 2 coordinated tasks, namely coverage control---consisting in spreading over a given environment---and formation control---consisting, in the case considered in this paper, in assembling a hexagonal formation (see \cite{cortes2017coordinated} for more details). To this end, we define the CBFs $h_{ij}$ for robot $i$ to execute task $j$ as in \cite{notomista2019constraint}. Finally, we consider the failure of 2 robots by setting their corresponding control inputs to 0 starting from iteration 180, for robot 6, and 240, for robot 2.

\begin{figure*}
\centering
\subfloat[Iteration 1]{\label{subfig:without1}\includegraphics[trim={7cm 3cm 6cm 2cm},clip,width=0.33\textwidth]{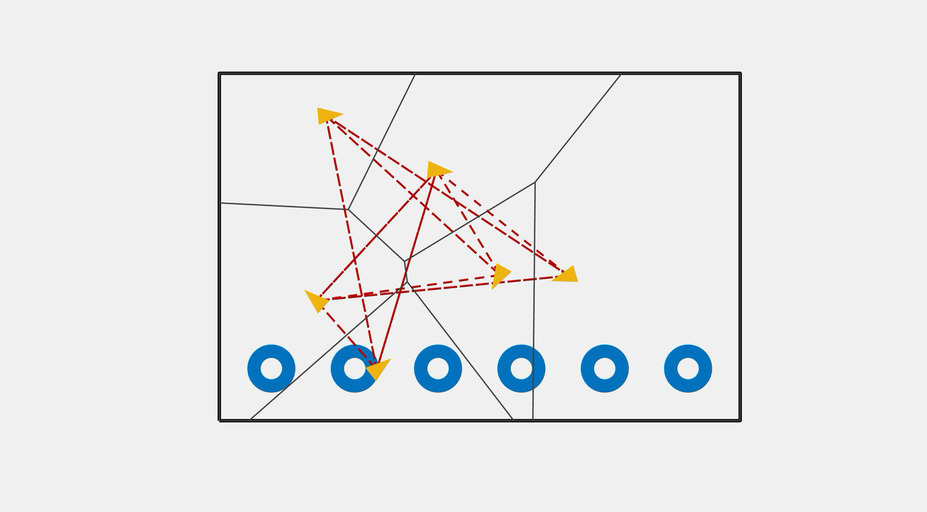}}%
\subfloat[Iteration 40]{\includegraphics[trim={7cm 3cm 6cm 2cm},clip,width=0.33\textwidth]{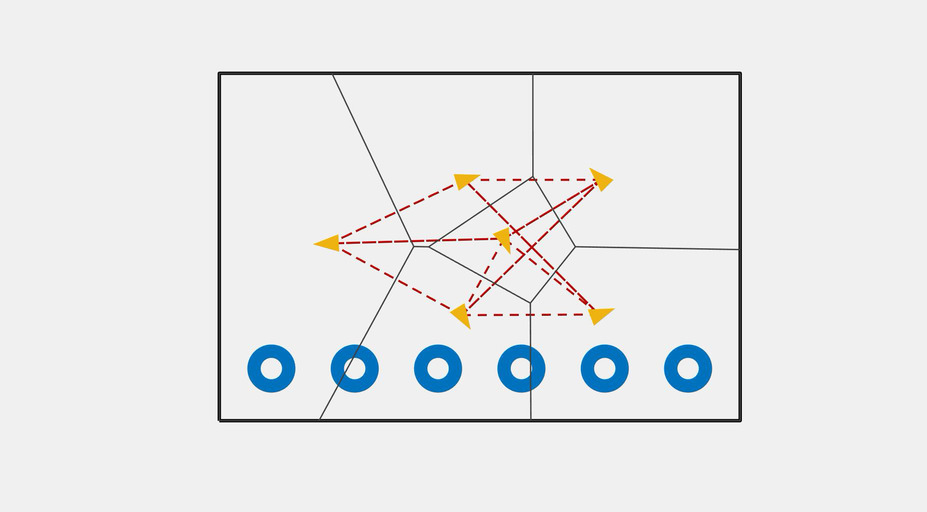}}%
\subfloat[Iteration 80]{\includegraphics[trim={7cm 3cm 6cm 2cm},clip,width=0.33\textwidth]{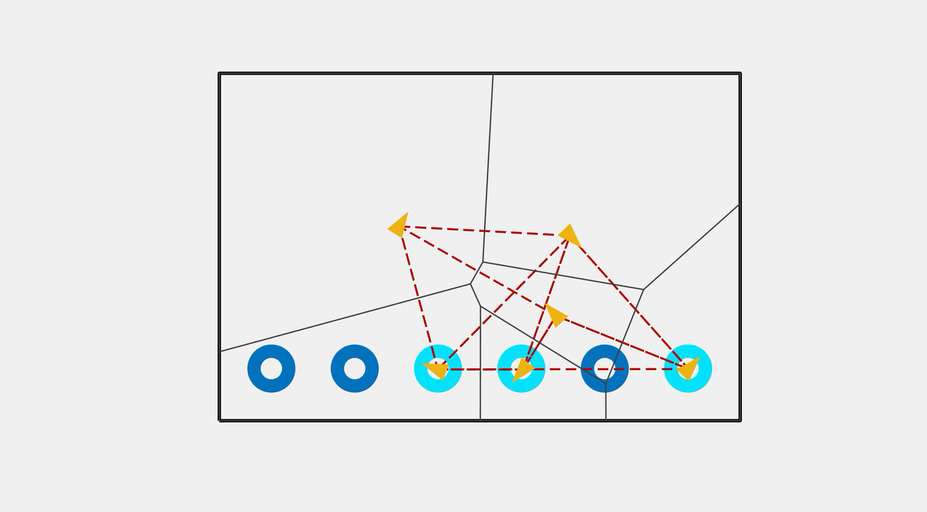}}\\
\subfloat[Iteration 120]{\includegraphics[trim={7cm 3cm 6cm 2cm},clip,width=0.33\textwidth]{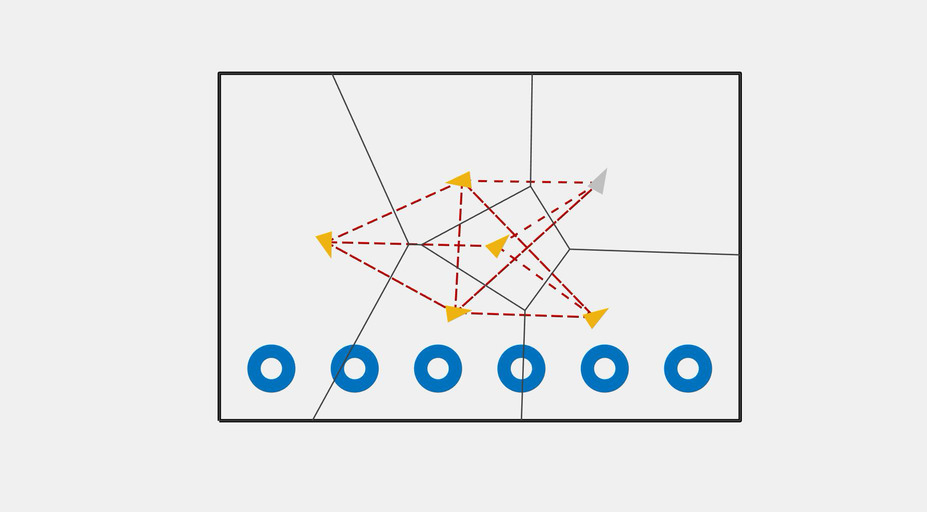}}%
\subfloat[Iteration 240]{\includegraphics[trim={7cm 3cm 6cm 2cm},clip,width=0.33\textwidth]{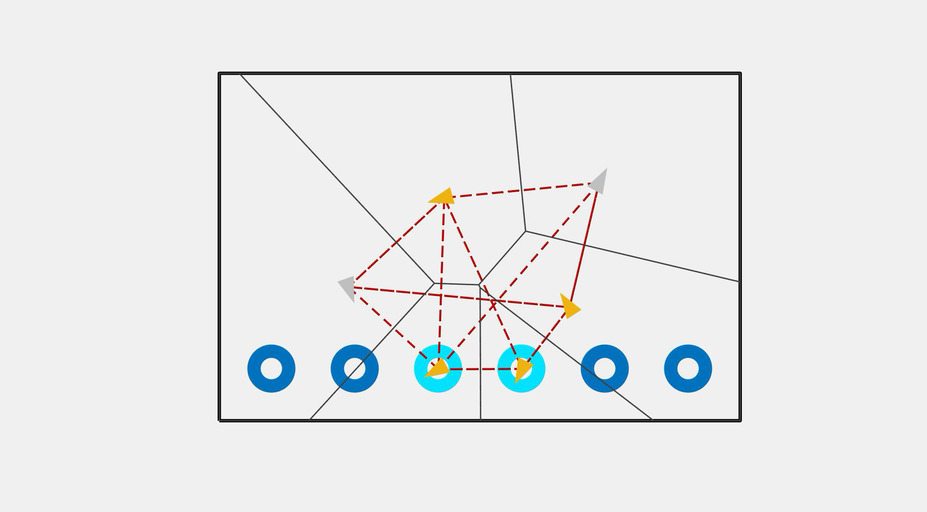}}%
\subfloat[Iteration 360]{\label{subfig:without2}\includegraphics[trim={7cm 3cm 6cm 2cm},clip,width=0.33\textwidth]{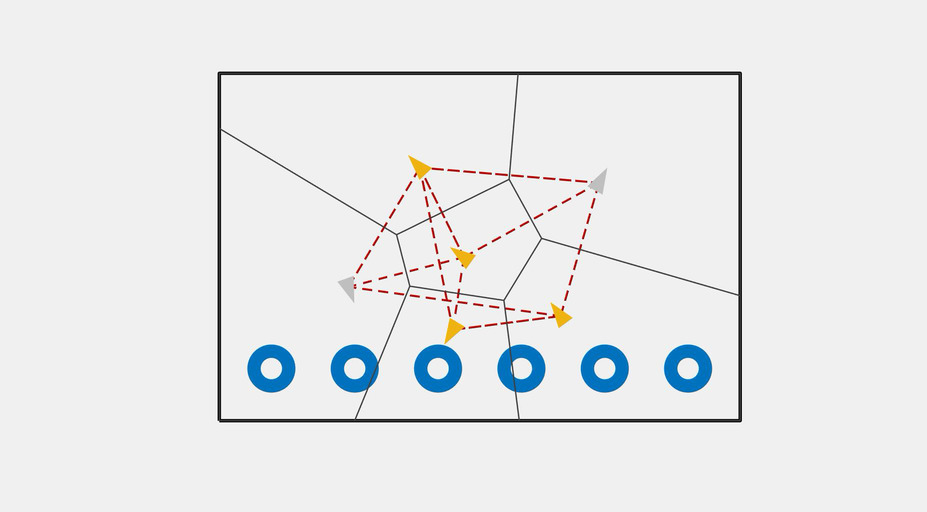}}
\caption{Snapshot from the simulated experiment where 6 robots (yellow triangles) are deployed to execute 2 tasks: coverage and formation control. The former is executed by achieving a centroidal Voronoi tessellation \cite{cortes2017coordinated}, while the latter consists in maintaining specified distances between pairs of robots. Voronoi cells are depicted as black solid lines, while pairs of robots maintaining specified distances are connected by a red dashed line segment. The robots execute the control input solution of \eqref{eq:cdcnrobotsmtaskscbfs}. As a result, they try to execute both tasks subject to the constraint that their energy never falls below a minimum threshold. In order to recharge their batteries, they are driven to dedicated charging stations (depicted as blue circles) by the control input solution of \eqref{eq:cdcnrobotsmtaskscbfs}. During the course of the experiment, two robots fail, after which they are not able to move anymore, and are shown as gray triangles.}
\label{fig:snapshotswithout}
\end{figure*}
\begin{figure*}
\centering
\subfloat[Iteration 1]{\label{subfig:with1}\includegraphics[trim={7cm 3cm 6cm 2cm},clip,width=0.33\textwidth]{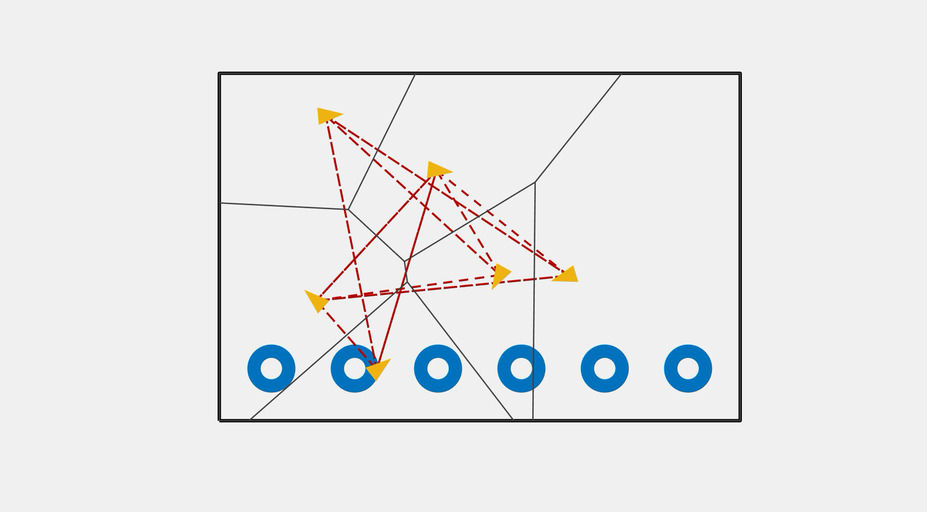}}%
\subfloat[Iteration 40]{\includegraphics[trim={7cm 3cm 6cm 2cm},clip,width=0.33\textwidth]{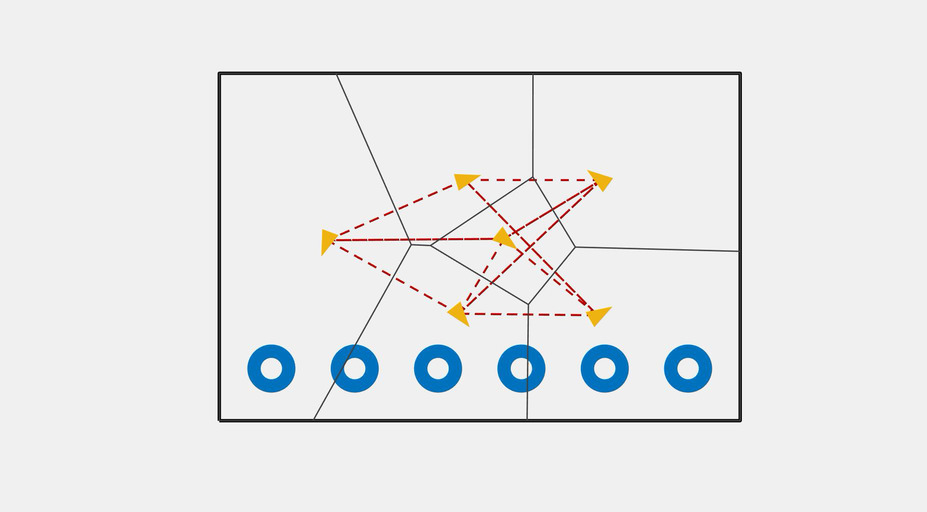}}%
\subfloat[Iteration 80]{\includegraphics[trim={7cm 3cm 6cm 2cm},clip,width=0.33\textwidth]{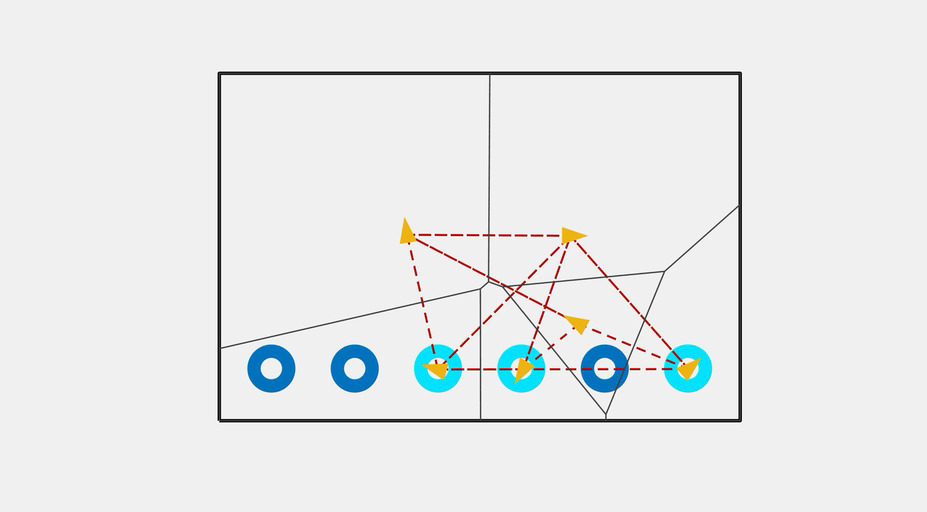}}\\
\subfloat[Iteration 120]{\includegraphics[trim={7cm 3cm 6cm 2cm},clip,width=0.33\textwidth]{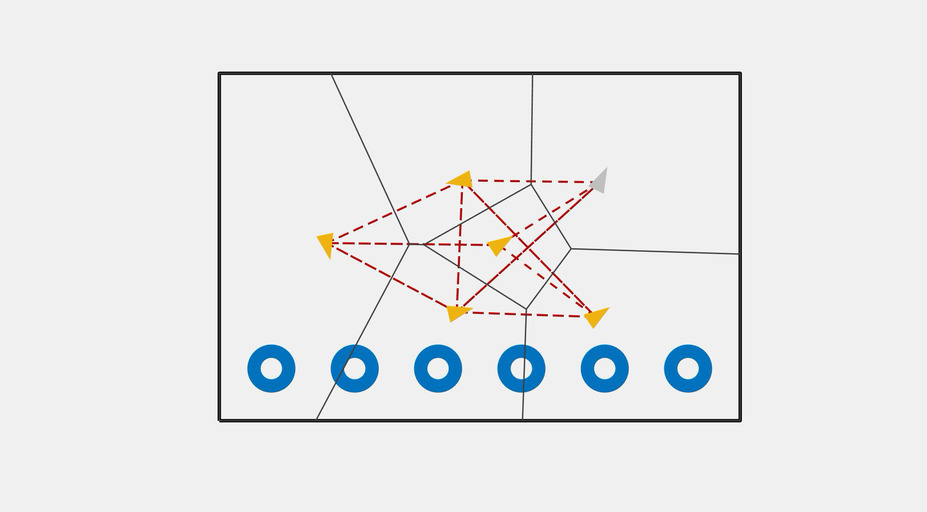}}%
\subfloat[Iteration 240]{\includegraphics[trim={7cm 3cm 6cm 2cm},clip,width=0.33\textwidth]{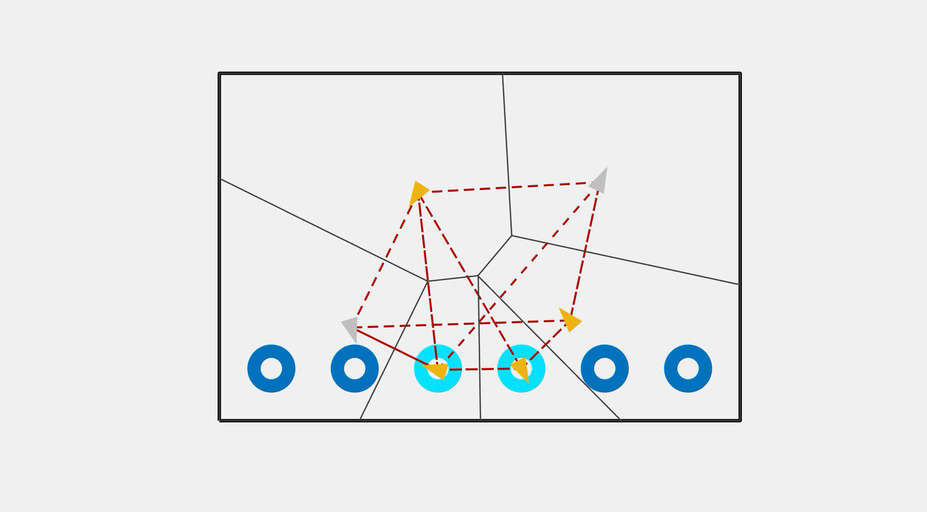}}%
\subfloat[Iteration 360]{\label{subfig:with2}\includegraphics[trim={7cm 3cm 6cm 2cm},clip,width=0.33\textwidth]{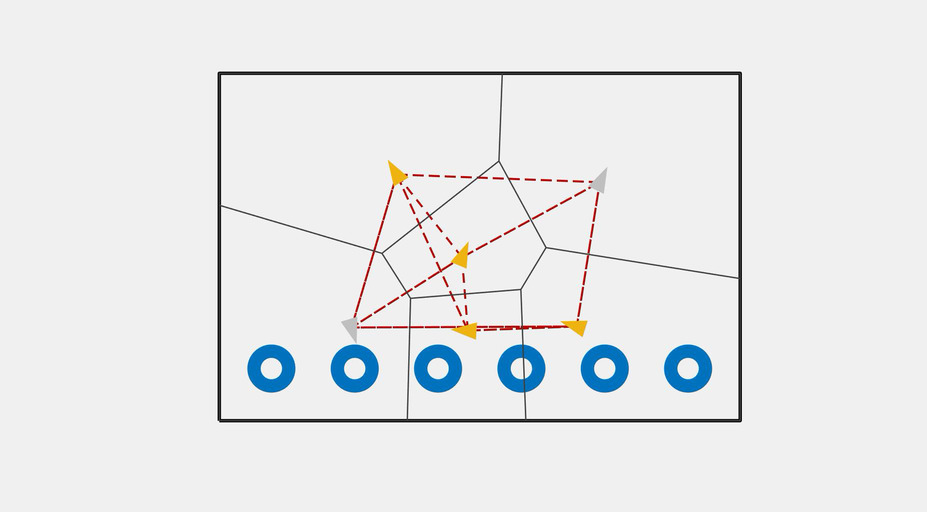}}
\caption{Same scenario of Fig.~\ref{fig:snapshotswithout}. In this case, however, the robots are driven by the control input solution of \eqref{eq:cdcnrobotsmtaskscbfsresilient}, i.e. with the addition of the resilient contraint \eqref{eq:rc}. The resulting final positions yield a better execution of both the coverage and the formation control tasks.}
\label{fig:snapshotswith}
\end{figure*}
Figures~\ref{fig:snapshotswithout} and \ref{fig:snapshotswith} show a sequence of snapshots recorded during the course of one simulated experiment in which the robots are controlled using the control input $u_i$ solution of the optimization program \eqref{eq:cdcnrobotsmtaskscbfs} (i.e. without the resilience constraint) and $u_i$ solution of the optimization program \eqref{eq:cdcnrobotsmtaskscbfsresilient} (i.e. with the resilience constraint). The robots are depicted as yellow triangles\footnote{The robots are modeled as single integrator, but the simulated dynamics are unicycle, i.e. points with an orientation. The transformation between single integrator inputs and unicycle inputs has been implemented as in \cite{notomista2019constraint} in order to control the robots in the simulator.}, which turn gray when the robots experience a failure. The blue circles represent the charging stations which turn light blue when the robots are charging on them. The black solid lines are the boundary of the Voronoi cells corresponding to the positions of the robots. These are used to execute the coverage control task as explained in \cite{cortes2017coordinated}. The red dashed lines are the edges between the robots on which specified distances are to be maintained in order to achieve the desired hexagonal formation.

\begin{figure}
\centering
\subfloat[Absolute value of the coverage control task CBF.]{\label{subfig:graphcoverage}\includegraphics[trim={2cm 0cm 3cm 0cm},clip,width=0.85\linewidth]{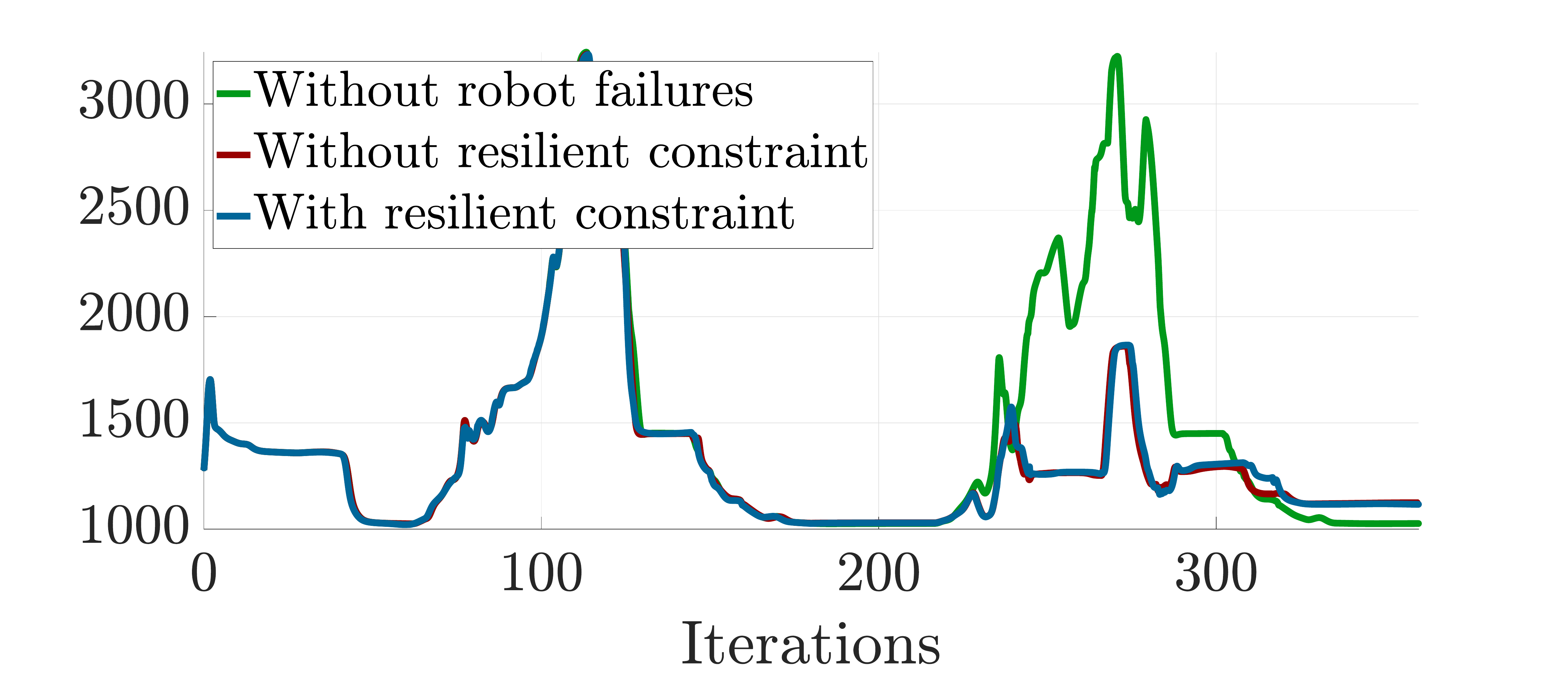}}\\
\subfloat[Absolute value of the formation control task CBF.]{\label{subfig:graphformation}\includegraphics[trim={3cm 0cm 3cm 0cm},clip,width=0.85\linewidth]{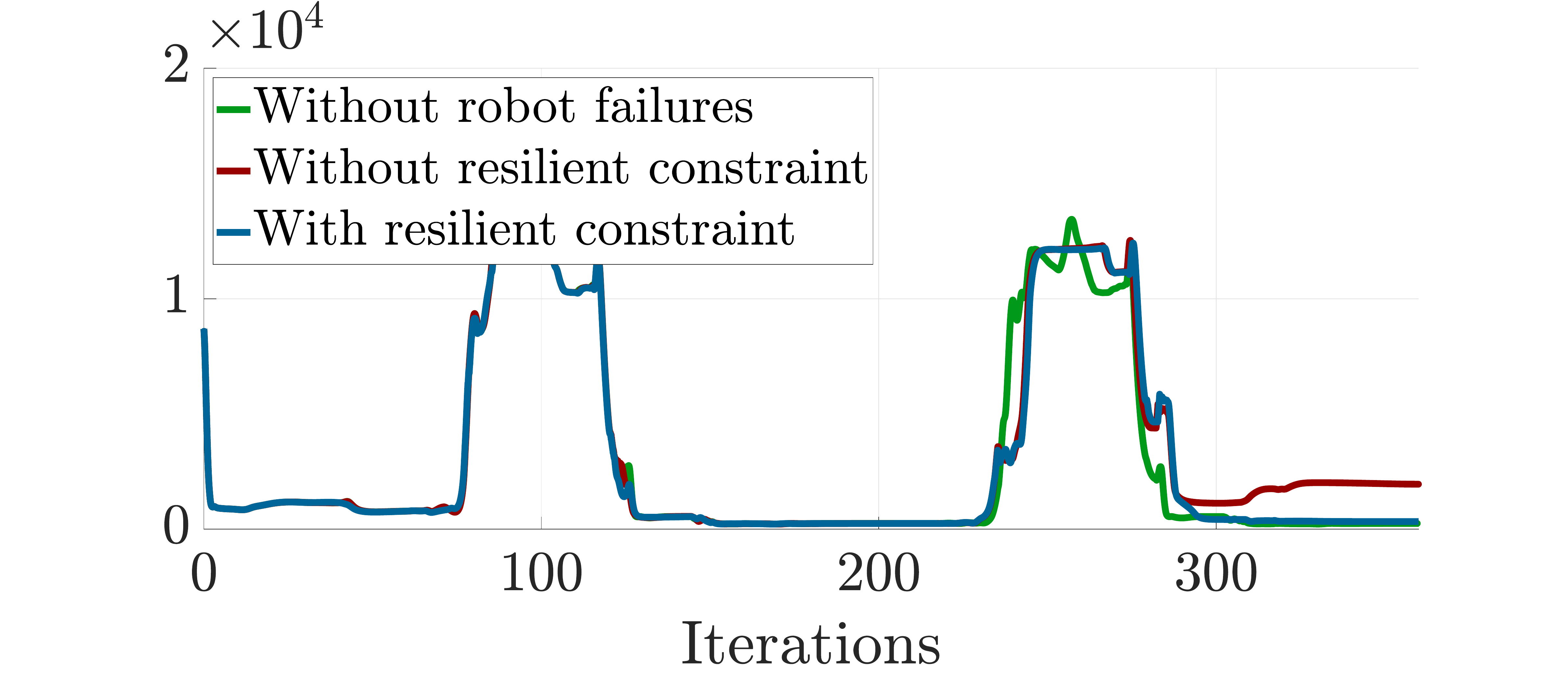}}
\caption{Absolute values of the task CBFs (averaged over 100 runs) recorded during the course of the simulated experiment. Lower values signify better task execution. Different colors correspond to baseline with no robot failures (green), experiment with robot failures but no resilient constraint (red), experiment with robot failures and resilient constraint (blue). As can be seen by the values of the task CBFs at the end of the simulation, adding the resilient constraint allows the robots to execute both tasks more effectively compared to the case where no resilient constraint is included.}
\label{fig:graphcomparison}
\end{figure}
Figure~\ref{fig:graphcomparison} reports the values of the task CBFs, corresponding to the coverage control task (Fig.~\ref{subfig:graphcoverage}) and the formation control task (Fig.~\ref{subfig:graphformation}), recorded during the course of the experiment. The values are the average over 100 simulated experiments. In fact, as derived in Section~\ref{subsec:analysis}, the resilience metric based on the frame potential holds in expectation over different tasks and failing robots. The baseline when no robot failures are introduced is plotted in green. The red line corresponds to the case where no resilience constraint is included in the optimization program \eqref{eq:cdcnrobotsmtaskscbfs}, while the blue line is obtained by letting the robots execute the control input solution of \eqref{eq:cdcnrobotsmtaskscbfsresilient}. The lower the value of task CBFs, the better the task is executed. As can be seen, adding the resilience constraints has the effect of allowing the robots to reconfigure themselves so that they are able to execute both tasks more effectively.

\begin{figure}
	\centering
	\includegraphics[trim={3cm 0cm 3cm 0cm},clip,width=0.85\linewidth]{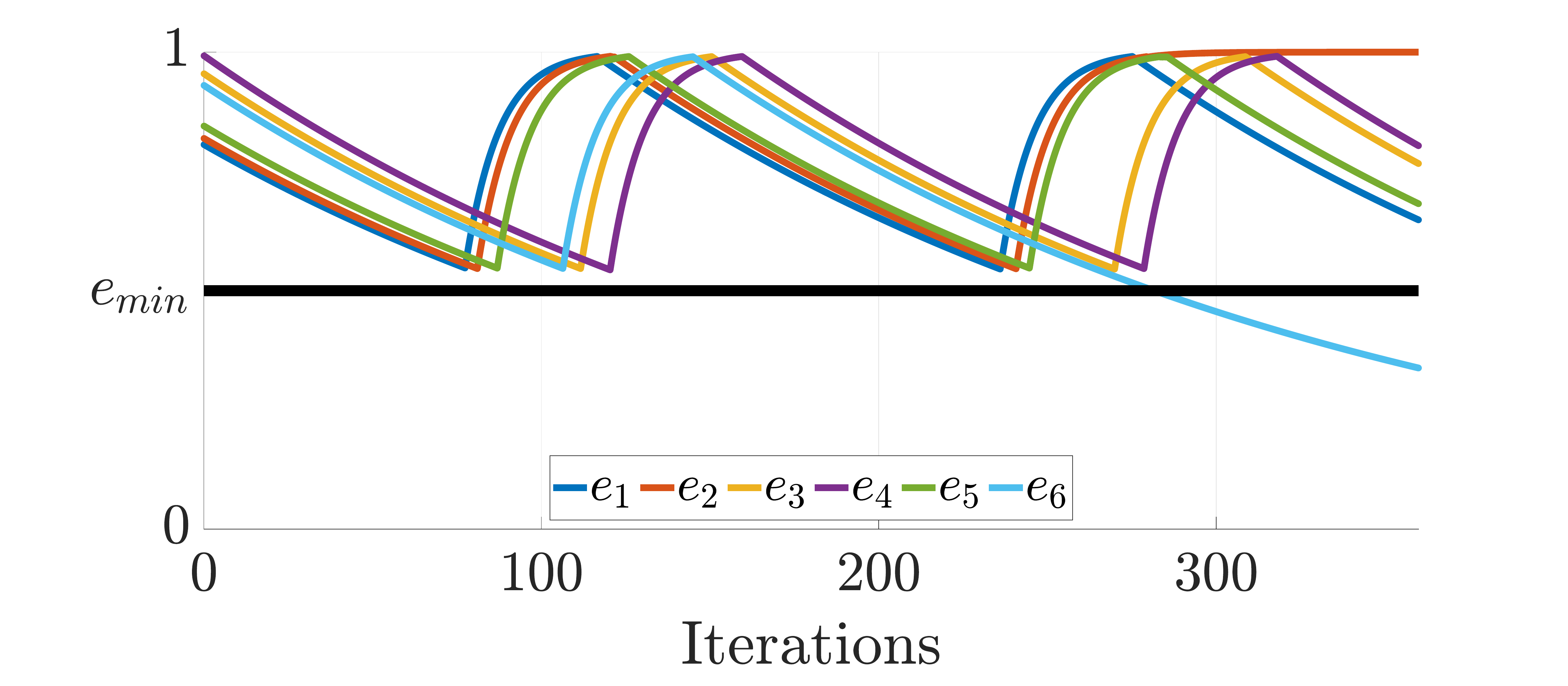}
	\caption{Energy stored in the batteries of the 6 robots executing 2 tasks in one simulated experiment. Thanks to constraint \eqref{eq:energyconstraint}, the value of the energy is kept above the minimum threshold $e_{min}$ (thick black line). The robots are controlled by the solution of the optimization program \eqref{eq:cdcnrobotsmtaskscbfsresilient} and drive to a dedicated charging station when their energy is getting too close to $e_{min}$. Except for two failing robots---whose energy cannot be controlled by driving to a charging station---all robots keep their energy above the minimum threshold $e_{min}$.}
	\label{fig:energy}
\end{figure}
Finally, Figure~\ref{fig:energy} shows the energy of the 6 robots recorded during the course of one simulated experiment. Thanks to the constraint \eqref{eq:energyconstraint}, the energy of all the robots---except for one robot that failed away from its charging station and therefore was not able to get back to it anymore---always remain above the minimum threshold $e_{min}$, whose value is marked with a thick black line. Thus, enforcing the resilient constraint \eqref{eq:rc} in the optimization-based control synthesis results in the multi-robot system being able to better execute both coverage and formation control, while, at the same time, allowing the robots to never discharge their batteries.

\begin{remark}
As can be seen from the snapshots of the simulations in Figures~\ref{fig:snapshotswithout} and \ref{fig:snapshotswith}, energy constraints move the robots away from the execution of the tasks. The resilience constraint also helps in the cases where the inputs of the robots deviate from the desired ones not because of failures, but rather because the robots are in need of energy. This is an additional benefit of considering energy awareness and resilience holistically for long-duration robot autonomy.
\end{remark}

\section{Conclusions}
\label{sec:conclusions}

In this paper, we presented a novel frame-theoretic metric for resilience of multi-robot systems, and an optimization-based control framework able to holistically ensure resilience and energy-awareness. A so-called frame potential is leveraged in the constraint-driven control framework in order to quantify and improve the resilience properties of multi-robot systems. The approach is showcased in simulation on a team of mobile robots executing multiple tasks and subject to unmodeled robot failures.


%%%%%%%%%%%%%%%%%%%%%%%%%%%%%%%%%%%%%%%%%%%%%%%%%%%%%%%%%%%%%%%%%%%%%%%%%%%%%%%%

%%%%%%%%%%%%%%%%%%%%%%%%%%%%%%%%%%%%%%%%%%%%%%%%%%%%%%%%%%%%%%%%%%%%%%%%%%%%%%%%

%%%%%%%%%%%%%%%%%%%%%%%%%%%%%%%%%%%%%%%%%%%%%%%%%%%%%%%%%%%%%%%%%%%%%%%%%%%%%%%%

%%%%%%%%%%%%%%%%%%%%%%%%%%%%%%%%%%%%%%%%%%%%%%%%%%%%%%%%%%%%%%%%%%%%%%%%%%%%%%%%

\bibliographystyle{IEEEtran}
\bibliography{bib/IEEEabrv,bib/references}

\end{document}